\documentclass[twoside]{article}

\usepackage[accepted]{aistats2026}

\usepackage[round]{natbib}

\usepackage[hidelinks]{hyperref}
\usepackage{amssymb}
\usepackage{amsmath}
\usepackage{amsthm}
\usepackage[utf8]{inputenc}
\usepackage{newunicodechar}
\newunicodechar{⁻}{\textsuperscript{-}}
\newtheorem{theorem}{Theorem}

\usepackage{algorithm}
\usepackage{algpseudocode}

\usepackage{graphicx}

\usepackage{booktabs}
\usepackage{multirow}
\usepackage{makecell}
\usepackage{xcolor}
\usepackage{cleveref}

\begin{document}
\runningtitle{SQuaT: Student-Aware Quantized Teacher Features}
\runningauthor{HyeonJun Lee, Hyeonsik Jo, Jinwoo Chung, Jangho Kim}
\twocolumn[
\aistatstitle{SQuaT: Self-Supervised Knowledge Distillation via \\ Student-Aware Quantized Teacher Features}
\aistatsauthor{HyeonJun Lee$^*$ \And Hyeonsik Jo$^*$ \And  Jinwoo Chung \And Jangho Kim$^{\dagger}$}
\aistatsaddress{Kookmin University \\ \{lbghj522, hsjo, imaboybut,  jangho.kim\}@kookmin.ac.kr} ]

\begin{abstract}
    Quantization-Aware Training (QAT) enables the deployment of quantized models with minimal accuracy degradation. However, in practical scenarios, training labels are often unavailable due to privacy, copyright, or cost constraints. Knowledge Distillation (KD) is a common approach to address this challenge, but we observe that prior work combining QAT with KD suffers from a fundamental limitation: during distillation, the range mismatch between the teacher and the quantized student model induces an unattainable residual, resulting in an irreducible lower bound on the distillation loss. Motivated by this observation, we propose \textbf{SQuaT} (\textbf{S}tudent-Aware \textbf{Qua}ntized \textbf{T}eacher Features), a label-free QAT framework with KD that theoretically eliminates this lower bound by applying the student’s quantization parameters to quantize the teacher’s features during distillation. Through comprehensive experiments across diverse settings, we demonstrate that SQuaT consistently outperforms strong baselines, with particularly pronounced gains in extreme low-bit 
    (e.g., 1- and 2-bit) settings. Furthermore, extensive evaluations across various model design choices show that our approach does not rely on specific architectural assumptions, making it broadly applicable across diverse architectures and quantization settings. The source code is available at \url{https://github.com/lcdbsa522/SQuaT}.
\end{abstract}

\begin{figure}[ht]
     \centering
     \includegraphics[width=0.45\textwidth]{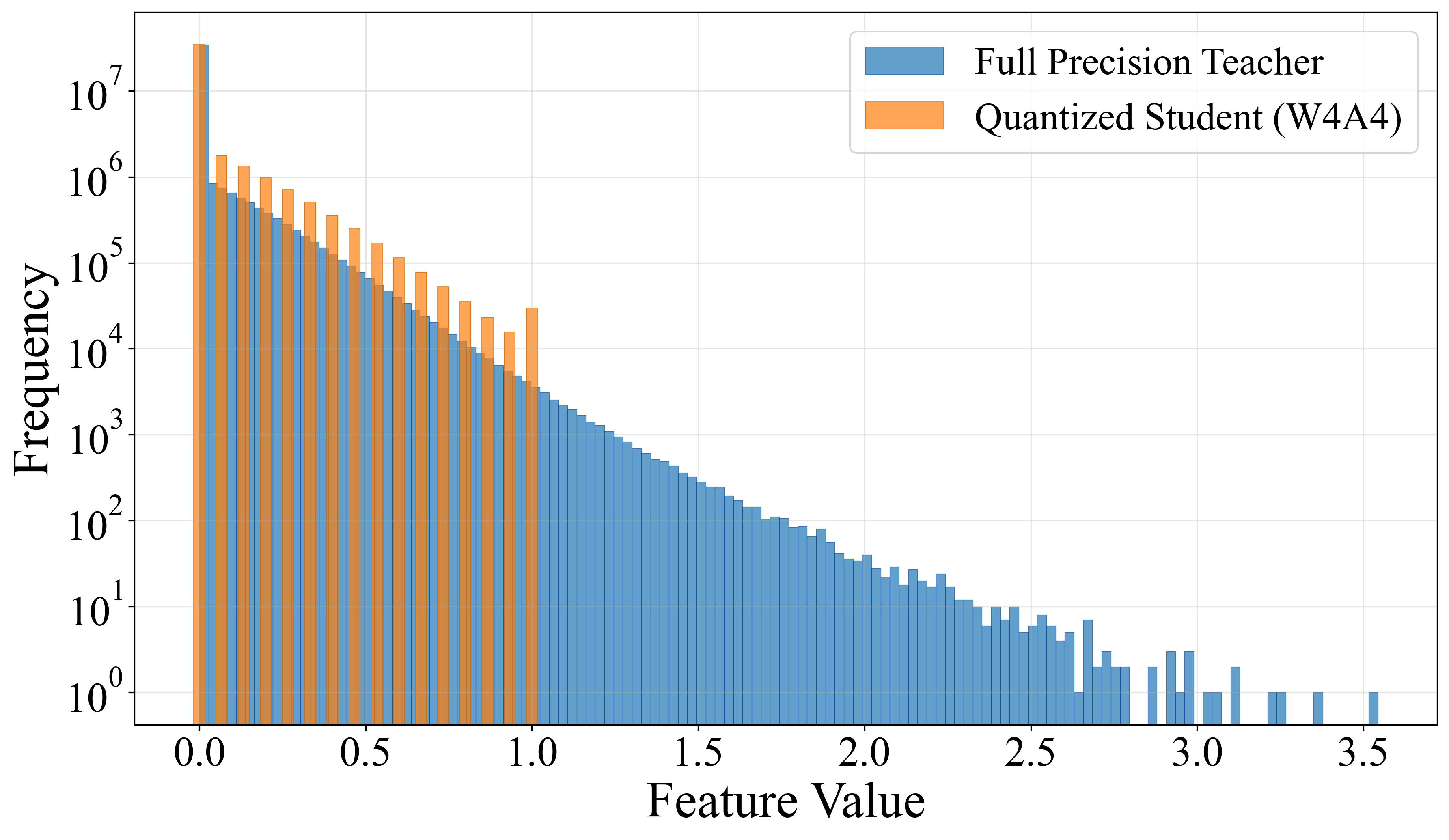}
     \caption{Distribution of feature maps for the FP teacher and 4-bit student in ResNet-20 on CIFAR-10.}
     \label{fig:Figure_1}
\end{figure}

\begin{figure*}[ht]
    \centering
    \includegraphics[width=0.8\textwidth]{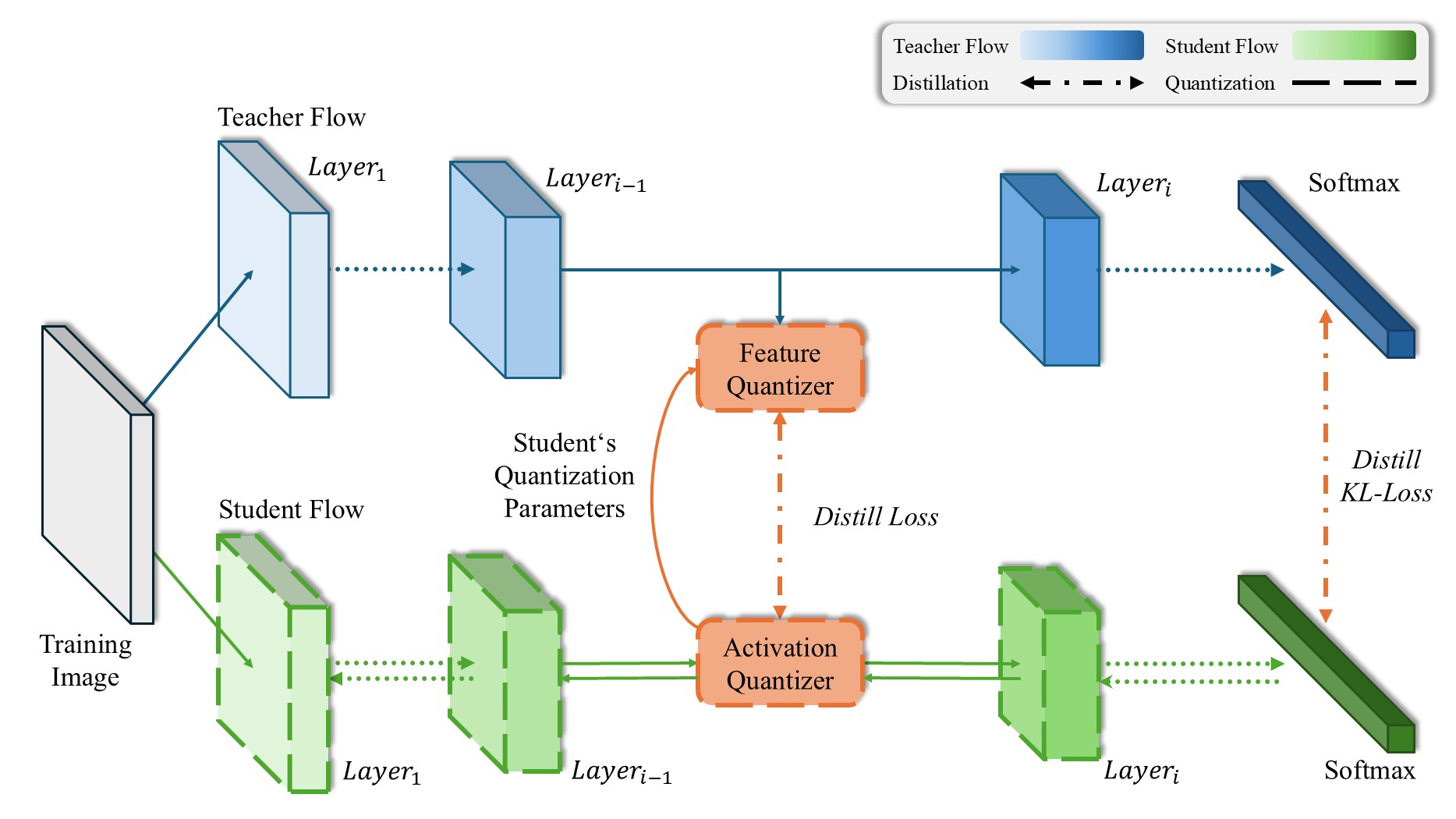}
    \caption{Overall framework of SQuaT. By projecting teacher features onto the student’s quantization lattice via a student-aware projection using the student’s quantization parameters, SQuaT reduces the distributional discrepancy between the teacher’s full-precision features and the student’s low-precision features, thereby enabling effective quantization-aware training without labels.}
    \label{fig:Figure_2}
\end{figure*}

\section{INTRODUCTION}
    Over the past decade, Deep Neural Networks (DNNs) have emerged as a dominant modeling paradigm in artificial intelligence, achieving remarkable performance across a wide range of domains, including computer vision and natural language processing. However, these advances have increasingly relied on larger model sizes, which incur substantial memory and computational costs at inference time, thereby posing practical challenges for deploying models in resource-constrained environments such as mobile and edge devices~\citep{han2015deep,kim2023finding,back2025magnitude}. Quantization is a widely adopted model compression technique that significantly reduces inference-time computation and memory usage by mapping high-precision floating-point weights and activations to low-precision integer values. In particular, Quantization-Aware Training (QAT) has been widely adopted as an effective strategy to mitigate the accuracy degradation of full-precision (FP) models by simulating quantization effects during training~\citep{jacob2018quantization, nagel2021white}. Nevertheless, existing QAT methods typically assume access to both the weights of an FP model and a sufficiently large amount of labeled training data. In practice, while the weights of pretrained FP models are often readily available, training labels are frequently inaccessible due to privacy, copyright, and cost constraints. These practical limitations naturally motivate Knowledge Distillation (KD)-based approaches~\citep{hinton2015distilling}, in which a pretrained FP model serves as a teacher model to guide a quantized student model. Accordingly, the central question addressed in this work is as follows: \textit{Can we perform effective and stable QAT relying solely on the teacher’s knowledge, without access to labeled data?}
    
    Self-Supervised Quantization-Aware Knowledge Distillation (SQAKD)~\citep{zhao2024self} demonstrated the feasibility of label-free QAT by relying exclusively on logit-level distillation during QAT training. However, this approach does not fully leverage the teacher’s intermediate representations, which contain substantially richer information than logits alone. Indeed, Self-Supervised Learning methods such as SimCLR~\citep{chen2020simple} and BYOL~\citep{grill2020bootstrap} show that feature-level learning enables powerful representation learning.
    
    However, directly applying feature-level knowledge distillation to a quantized student model in this manner suffers from a fundamental limitation. As illustrated in Figure~\ref{fig:Figure_1}, the intermediate feature maps of a FP teacher model and a low-bit quantized student model have markedly different value ranges. In other words, the representation manifolds of the FP teacher and the low-bit quantized student are inherently mismatched. While the student’s outputs are constrained to lie on a quantization lattice defined by a clipping range and step size determined by the target bit-level, the teacher produces unconstrained full-precision features. Consequently, using FP features as distillation targets for the student induces an unattainable residual between the teacher and the student, resulting in a theoretically irreducible lower bound on the distillation loss. This lower bound forces the student to continually pursue targets that it cannot attain, leading to distorted gradient signals and ultimately unstable optimization.
    
    In a related line of work, Quantized Feature Distillation (QFD)~\citep{zhu2023quantized}, proposed in a supervised learning setting, attempts to perform feature-level knowledge distillation tailored to quantized models. QFD introduces a feature quantizer into the full-precision teacher model prior to QAT and calibrates it through a few epochs of pretraining—typically corresponding to a small fraction of the total QAT training—in order to preserve the teacher’s performance. During subsequent QAT, the calibrated quantizer is used to quantize the teacher’s features to the same target bit-level as the student, enabling feature distillation. However, this approach introduces additional overhead, as it requires additional pretraining of a teacher-side quantizer. More importantly, because the teacher is calibrated independently of the student’s quantization characteristics, the quantization lattice of the teacher does not align with that of the student, even at the same bit-level. As a result, the semantic meaning of quantized values may differ between the teacher and the student.
    
    To address this fundamental limitation, we propose SQuaT (Student-Aware Quantized Teacher Features). The core idea of SQuaT is to quantize the intermediate features of the teacher model using the quantization parameters of the student model during distillation, thereby projecting the teacher’s features directly into the student’s quantized representation space. This student-aware projection aligns the teacher and student features on the same quantization lattice and enables feature-level knowledge distillation using only targets that are attainable by the student. As a result, the unattainable residual induced by FP targets that the student cannot represent is eliminated, thereby removing the theoretical lower bound on the distillation loss. We theoretically analyze and formally prove the removal of this lower bound in Section~\ref{section:3}.
    
    Furthermore, by directly reusing the student model’s quantization parameters, SQuaT ensures that the teacher and student features are interpreted consistently within the student’s quantized representation space. In other words, the semantic meaning of values represented on the quantization lattice is preserved across the two models from the student’s perspective. Figure~\ref{fig:Figure_2} provides an overview of the complete SQuaT distillation framework.
    
    While prior feature-level knowledge distillation approaches have relied on teacher-aware projection, which maps the student toward the teacher’s representation space, we instead introduce a paradigm shift to student-aware projection, in which the teacher is projected into the student’s representation space. Through this shift, SQuaT enables stable and effective feature-level knowledge distillation in label-free QAT settings.
    
    Our main contributions are summarized as follows:
    \begin{itemize}
        \item We propose SQuaT, which introduces a paradigm shift from teacher-aware to student-aware projection by leveraging the student model’s quantization parameters.
        \item Through theoretical and empirical analyses, we show that SQuaT effectively removes the lower bound caused by unattainable residuals in prior teacher-aware projection-based methods.
        \item We conduct extensive experiments to validate the proposed approach, demonstrating that SQuaT is agnostic to model design choices and broadly applicable, with particularly pronounced gains in extreme low-bit settings.
    \end{itemize}

\section{RELATED WORK}
    \subsection{Quantization-Aware Training}
        Quantization-Aware Training (QAT) simulates quantization-induced errors during training, enabling models to adapt to low-precision computation environments~\citep{jacob2018quantization, nagel2021white}. In the forward pass, weights and activations are quantized for computation, while gradients are computed with respect to the underlying full-precision (FP) weights in the backward pass. To enable gradient propagation, approximation techniques such as the Straight-Through Estimator (STE)~\citep{bengio2013estimating} have been widely adopted; however, their limitations have also been extensively discussed in prior work~\citep{bai2018proxquant,chen2019metaquant,gong2019differentiable}. To improve the stability and accuracy of QAT, various methods have been proposed~\citep{choi2018pact, li2019additive, uhlich2019differentiable}. Among them, Learned Step Size Quantization (LSQ) introduces trainable step sizes~\citep{esser2019lsq}, while Element-Wise Gradient Scaling (EWGS) leverages gradient scaling to alleviate STE-induced instability~\citep{lee2021network}. These advances have enabled QAT to be extended to extreme low-bit settings with improved stability and accuracy.

    \subsection{Knowledge Distillation for QAT}
        Knowledge distillation (KD) is a widely adopted approach for mitigating performance degradation in low-bit QAT settings. Quantized Feature Distillation (QFD)~\citep{zhu2023quantized} performs feature-level KD in supervised QAT environments. Specifically, prior to QAT, feature quantizers are inserted into intermediate layers of the teacher network and trained to compensate for quantization loss. During the subsequent QAT stage, these calibrated quantizers convert teacher features to the target bit-width, and distillation is performed on the quantized representations. Although this design attempts to align the distillation signal with the low-bit representation space, it introduces additional training overhead due to the calibration of teacher-side quantizers. Moreover, since the teacher’s quantization parameters are optimized independently of the student’s, their quantization lattices may not coincide even at the same bit-width, potentially resulting in misalignment between teacher and student representations.
    
    \subsection{Self-Supervised Learning for QAT}
        As labeled data become increasingly difficult to obtain due to privacy and cost constraints, self-supervised learning has attracted increasing attention. Self-Supervised Quantization-Aware Knowledge Distillation (SQAKD)~\citep{zhao2024self} demonstrated the feasibility of label-free QAT through logit-level distillation. However, it relies solely on output probability distributions and does not fully exploit the rich intermediate representations of the teacher model. Self-supervised learning methods such as SimCLR~\citep{chen2020simple} and BYOL~\citep{grill2020bootstrap} have shown that feature-level learning plays a crucial role in learning discriminative representations. This suggests that incorporating feature-level guidance, in addition to logit-level supervision, may help mitigate representation distortion caused by quantization in QAT settings.

\section{PROPOSED FRAMEWORK}
\label{section:3}
    In this section, we first review the formulation of QAT, and then introduce SQuaT, a self-supervised feature distillation framework. The central idea of SQuaT is to project teacher features onto the student’s quantization lattice, thereby enabling joint alignment of feature- and logit-level signals within a unified representation space.

    \subsection{Quantization-Aware Training}
        We adopt the EWGS method~\citep{lee2021network}, which directly learns quantization parameters rather than relying on the simple STE-based QAT formulation~\citep{bengio2013estimating}. We now describe the forward and backward quantization procedures used in SQuaT.
        
        \subsubsection{Forward pass}
        \label{section:3.1.1}
            Given an input tensor $x$ (representing either weights or activations), a fixed global bit-width $b$, and learnable normalization range parameters $(a^-,a^+)$, the forward quantization proceeds as follows. First, normalize to $[0,1]$:
            \begin{equation}
            x_{\mathrm{norm}}=\mathrm{clip}\!\left(\frac{x-a^-}{\,a^+-a^-\,},\,0,\,1\right).
            \end{equation}
            Project onto the $b$-bit uniform grid:
            \begin{equation}
            x_{\mathrm{quant}}=\frac{\mathrm{round}\!\left(x_{\mathrm{norm}}\cdot(2^b-1)\right)}{2^b-1}\ \in\ \Big\{0,\frac{1}{2^b-1},\dots,1\Big\}.
            \end{equation}
            Apply a type-specific linear mapping (as in EWGS) to obtain the quantized output:
            \begin{equation}
                \widehat x_{\phi_\ell}=
                \begin{cases}
                    x_{\mathrm{quant}}, & \text{activations},\\[2pt]
                    2\!\left(x_{\mathrm{quant}}-0.5\right), & \text{weights}.
                \end{cases}
            \end{equation}
            Here, $\phi_\ell$ denotes the quantization parameters for layer $\ell$, including the range $(a^-_\ell,a^+_\ell)$ and the bit-width $b_\ell$, i.e., $\phi_\ell=\{a^-_\ell,a^+_\ell,b_\ell\}$. These parameters can be defined either per layer or per channel.
        
        \subsubsection{Backward pass}
            Quantization involves non-differentiable operations (e.g., \texttt{round}, \texttt{clip}). Following EWGS, gradients are rescaled based on the quantization error and upstream signals. Let $x_{\mathrm{norm},\ell}$ be the normalized value, $x_{\mathrm{quant},\ell}$ the discretized value at layer $\ell$, $e_\ell=x_{\mathrm{norm},\ell}-x_{\mathrm{quant},\ell}$ the quantization error, and $g_\ell=\partial \mathcal L/\partial \widehat x_{\phi_\ell}$ the upstream gradient. Then
            \begin{align}
            \begin{split}
                g_{q,\ell}=\frac{\partial \mathcal L}{\partial x_{\mathrm{quant},\ell}}
                &=\frac{\partial \mathcal L}{\partial \widehat x_{\phi_\ell}}\cdot
                 \frac{\partial \widehat x_{\phi_\ell}}{\partial x_{\mathrm{quant},\ell}}\\
                &=
                \begin{cases}
                    g_\ell, & \text{activations},\\[2pt]
                    2g_\ell, & \text{weights},
                \end{cases}
            \end{split}
            \end{align}
            \begin{equation}
                \frac{\partial \mathcal L}{\partial x_{\mathrm{norm},\ell}}
                \ \approx\ 
                g_{q,\ell}\!\left(1+\eta\cdot\mathrm{sign}(g_{q,\ell})\cdot e_\ell\right),
            \end{equation}
            where $\eta$ denotes a small scaling hyperparameter.
            
    \subsection{SQuaT}
        For each layer $\ell$, let $f_{S,\ell}(\cdot;\theta_S,\phi_{S,\ell})$ and $f_{T,\ell}(\cdot)$ denote the student and teacher feature extractors at layer $\ell$, respectively, where $\phi_{S,\ell}=\{a^-_\ell,a^+_\ell,b_\ell\}$ are the student’s quantization parameters for that layer.
        
        \subsubsection{Student-Aware Projection}
            We define the student-aware projection at layer $\ell$ by
            \begin{equation}
                \Pi_{\phi_{S,\ell}}(z) \;\triangleq\; \widehat z_{\phi_{S,\ell}},
            \end{equation}
            where $\widehat z_{\phi_{S,\ell}}$ is obtained from the student's forward quantization path at layer $\ell$ (normalize $\rightarrow$ grid projection $\rightarrow$ type-specific mapping), as defined in Section~\ref{section:3.1.1}.
            
            We project teacher features onto the student’s quantization lattice via a student-aware projection, yielding student-aware quantized teacher features. The student is then trained to match these projected targets. Let $\mathcal K$ be the set of distillation layers. Define $\tilde f_{T,\ell}(x)\triangleq \Pi_{\phi_{S,\ell}}\!\big(f_{T,\ell}(x)\big)$ and optimize the following objective:
            \begin{equation}
                \min_{\theta_S,\{\phi_{S,\ell}\}}\ 
                \mathbb{E}_{x\sim\mathcal D}\,
                \frac{1}{|\mathcal K|}\sum_{\ell\in \mathcal K}
                \big\|\,\widehat f_{S,\ell}(x;\theta_S,\phi_{S,\ell})-\tilde f_{T,\ell}(x)\,\big\|_2^2,
            \end{equation}
            where the student-side feature corresponds to the actual (already quantized) student output:
            \begin{equation}
                \widehat f_{S,\ell}(x;\theta_S,\phi_{S,\ell}) \;\triangleq\; f_{S,\ell}(x;\theta_S,\phi_{S,\ell}).
            \end{equation}

        \subsubsection{Self-Supervised Knowledge Distillation}
            SQuaT performs label-free joint alignment at both the feature and logit levels. These objectives are complementary: the former preserves spatial/channel patterns, while the latter maintains class-level semantics.
            \paragraph{Feature-level alignment.}
                For a set of distillation layers $\mathcal K$, we align quantized student  features $\widehat f_{S,\ell}$ with SQuaT targets $\tilde f_{T,\ell}$:
                \begin{equation}
                    \mathcal L_{\mathrm{feat}}
                    \;=\;
                    \frac{1}{|\mathcal K|}
                    \sum_{\ell\in\mathcal K}
                    \left\Vert
                        \widehat f_{S,\ell}(x;\theta_S,\phi_{S,\ell}) - \tilde f_{T,\ell}(x)
                    \right\Vert^2_2,
                \end{equation}
            \paragraph{Logit-level alignment.}
                We smooth predictions with temperature $T_{kd}\!>\!0$ and apply KL divergence:
                \begin{align}
                \begin{split}
                    p_T=\mathrm{softmax}\!\left(\frac{z_T}{T_{kd}}\right),\quad
                    p_S=\mathrm{softmax}\!\left(\frac{z_S}{T_{kd}}\right),
                \end{split}
                \end{align}
                \begin{equation}
                    \mathcal L_{\mathrm{logit}}
                    \;=\;
                    T_{kd}^{2}\cdot \mathrm{KL}\!\left(p_T \,\Vert\, p_S\right),
                \end{equation}
                where $p_T$ serves as a pseudo-label distribution that preserves inter-class relationships and mitigates semantic drift that may arise from feature-only alignment.
            \paragraph{Joint objective.}
                By replacing the unattainable target $f_{T,\ell}$ with the attainable $\tilde f_{T,\ell}$, SQuaT minimizes
                \begin{equation}
                    \mathcal L_{\mathrm{SQuaT}}
                    \;=\;
                    \lambda_{\mathrm{feat}}\,\mathcal L_{\mathrm{feat}}
                    \;+\;
                    \lambda_{\mathrm{logit}}\,\mathcal L_{\mathrm{logit}},
                \end{equation}
                where $\lambda_{\mathrm{feat}},\lambda_{\mathrm{logit}}\!\ge\!0$ balance feature- and logit-level objectives, and the feature targets are constructed via the student-aware projection.

                The overall process of SQuaT is depicted in Algorithm~\ref{alg:algorithm_1}. During training, the teacher network $T$ remains fixed, and only the student parameters $\theta_S$ and quantization parameters $\phi_S$ are updated.

    \begin{algorithm}[t]
        \caption{\textbf{SQuaT: Self-Supervised Knowledge Distillation via Student-Aware Quantized Teacher Features}}
        \label{alg:algorithm_1}
        \begin{algorithmic}[1]
            \Require Frozen teacher $f_T(\cdot;\theta_T)$, quantized student $f_S(\cdot;\theta_S,\{\phi_{S,\ell}\}_{\ell\in L})$, unlabeled data $D$, distillation layers $\mathcal K$, temperature $T_{\mathrm{kd}}$, loss weights $(\lambda_{\mathrm{feat}},\lambda_{\mathrm{logit}})$, step size $\eta$
            \Ensure Updated student parameters $(\theta_S,\{\phi_{S,\ell}\}_{\ell\in L})$
            \Statex \textit{Notation:} $\tilde f_{T,\ell}(x) \triangleq \Pi_{\phi_{S,\ell}}(f_{T,\ell}(x))$, where $\Pi_{\phi_{S,\ell}}$ projects the teacher features onto the student’s quantization lattice via a student-aware projection.
        
            \For{each training step $t=1,2,\dots$}
                \State Sample a mini-batch $x \sim D$
                \For{each $\ell \in \mathcal K$}
                    \State $\tilde f_{T,\ell}(x) \leftarrow \Pi_{\phi_{S,\ell}}\!\big(f_{T,\ell}(x)\big)$
                \EndFor
        
                \State $\displaystyle \mathcal{L}_{\mathrm{feat}} \leftarrow \frac{1}{|\mathcal K|}
                    \sum_{\ell\in\mathcal K}
                    \left\Vert
                        \widehat f_{S,\ell}(x;\theta_S,\phi_{S,\ell}) - \tilde f_{T,\ell}(x)
                    \right\Vert^2_2$
        
                \State $\displaystyle \mathcal{L}_{\mathrm{logit}} \leftarrow T_{\mathrm{kd}}^2 \cdot
                \mathrm{KL}\!\big(p_T \,\|\, p_S\big)$
        
                \State $\displaystyle \mathcal{L} \leftarrow \lambda_{\mathrm{feat}}\,\mathcal{L}_{\mathrm{feat}}
                + \lambda_{\mathrm{logit}}\,\mathcal{L}_{\mathrm{logit}}$
        
                \State $\theta_S \leftarrow \theta_S - \eta \,\nabla_{\theta_S}\mathcal{L}$
                \State $\phi_{S,\ell} \leftarrow \phi_{S,\ell} - \eta \,\nabla_{\phi_{S,\ell}}\mathcal{L}\quad\forall \ell\in L$
            \EndFor
        \end{algorithmic}
    \end{algorithm}

    \subsection{Analysis of Student-Aware Projection}
        \label{section:3.3}
        \begin{theorem}
        \label{theorem:1}
            Fix a layer $\ell$. Let $\Pi_{\phi_{S,\ell}}:\mathbb{R}^{d_\ell}\to\mathbb{R}^{d_\ell}$ denote the student-aware projection (the student's forward quantization path at layer $\ell$), and let $\mathcal G_{\phi_{S,\ell}} \triangleq \mathrm{Im}(\Pi_{\phi_{S,\ell}})$ be the set of outputs attainable by the student at this layer. For the teacher feature $f_{T,\ell}$ (random vector induced by $x\!\sim\!\mathcal D$), the following hold:
            \begin{align}
                \text{\normalfont(a)}\quad
                &\inf_{h \in \mathcal G_{\phi_{S,\ell}}}\;
                \mathbb{E}\big\|h - f_{T,\ell}\big\|_2^2
                \;=\;
                \mathbb{E}\big\|\Pi_{\phi_{S,\ell}}(f_{T,\ell}) - f_{T,\ell}\big\|_2^2,
                \label{eq:equation_13} \\[4pt]
                \text{\normalfont(b)}\quad
                &\inf_{h \in \mathcal G_{\phi_{S,\ell}}}\;
                \mathbb{E}\big\|h - \Pi_{\phi_{S,\ell}}(f_{T,\ell})\big\|_2^2
                \;=\; 0.
                \label{eq:equation_14}
            \end{align}
        \end{theorem}
        
        \begin{proof}
            We first establish the result pointwise and subsequently take expectations.
            
            (a) For any realization of $f_{T,\ell}$, consider the least-squares problem $\min_{h\in\mathcal G_{\phi_{S,\ell}}}\|h - f_{T,\ell}\|_2^2$. By definition, $\mathcal G_{\phi_{S,\ell}}$ is precisely the range of the student’s forward quantization map $\Pi_{\phi_{S,\ell}}$; therefore, the nearest attainable point to $f_{T,\ell}$ is its image under this mapping, $\Pi_{\phi_{S,\ell}}(f_{T,\ell})$. Hence
            \[
            \inf_{h\in\mathcal G_{\phi_{S,\ell}}}\|h - f_{T,\ell}\|_2^2
            = \big\|\Pi_{\phi_{S,\ell}}(f_{T,\ell}) - f_{T,\ell}\big\|_2^2.
            \]
            Taking expectations over $x\!\sim\!\mathcal D$ yields \eqref{eq:equation_13}.
            
            (b) Let the target be $\Pi_{\phi_{S,\ell}}(f_{T,\ell})$. Choosing $h=\Pi_{\phi_{S,\ell}}(f_{T,\ell})\in\mathcal G_{\phi_{S,\ell}}$ achieves zero error pointwise, which implies
            \(
            \inf_{h\in\mathcal G_{\phi_{S,\ell}}}\|h - \Pi_{\phi_{S,\ell}}(f_{T,\ell})\|_2^2 = 0.
            \)
            Taking expectations gives \eqref{eq:equation_14}.
        \end{proof}
        
        Consequently, supervising a quantized student with the FP target $f_{T,\ell}$ cannot drive the expected squared error below the quantization mismatch $\mathbb{E}\big\|\Pi_{\phi_{S,\ell}}(f_{T,\ell})-f_{T,\ell}\big\|_2^2$, whereas using the student-quantized target $\tilde f_{T,\ell}=\Pi_{\phi_{S,\ell}}(f_{T,\ell})$ removes this lower bound. Moreover, letting $\mathrm{dist}(u,S)$ denote the Euclidean distance, we have
        \[
        \mathbb{E}\Big[\big\|\Pi_{\phi_{S,\ell}}(f_{T,\ell})-f_{T,\ell}\big\|_2^2\Big]
        \;=\;
        \mathbb{E}\Big[\mathrm{dist}\!\big(f_{T,\ell},\mathcal G_{\phi_{S,\ell}}\big)^2\Big],
        \]
        so whenever any component of $f_{T,\ell}$ lies outside the learned activation range $[a^-_\ell,a^+_\ell]$ (or outside the weight mapping range), this lower bound \emph{increases monotonically} with the distance beyond the range. Projecting the teacher feature via $\Pi_{\phi_{S,\ell}}$ maps such components back onto the student’s attainable set, thereby removing the unattainable residual and stabilizing optimization.

\section{EXPERIMENTS}
    In this section, we evaluate SQuaT across diverse datasets, model architectures, and quantization settings. We first compare SQuaT with the supervised QAT method EWGS~\citep{lee2021network} and the label-free QAT with KD baseline SQAKD~\citep{zhao2024self} on CIFAR-10 and CIFAR-100~\citep{krizhevsky2009learning}. For the large-scale ImageNet-1K benchmark~\citep{deng2009imagenet}, we compare SQuaT with SQAKD. To evaluate architectural generalization in vision models, we additionally conduct experiments on the Vision Transformer--based DeiT-Tiny~\citep{dosovitskiy2020image, touvron2021training} for image classification on CIFAR datasets. For NLP tasks, we evaluate BERT~\citep{devlin2019bert} on the GLUE benchmark~\citep{wang2018glue}. Furthermore, we compare projection strategies for feature distillation to empirically validate the student-aware projection principle described in Section~\ref{section:3.3}. Finally, we measure on-device inference performance on edge hardware to assess practical deployability.
    
    \subsection{Experimental Setup}
        We describe the datasets, model architectures, and quantization configurations used in our experiments. For vision tasks, we conduct image classification on CIFAR-10, CIFAR-100, and ImageNet-1K. For NLP tasks, we use the GLUE benchmark, including RTE, SST-2, QNLI. All experiments are performed in a strictly label-free setting without access to ground-truth labels, and the pretrained full-precision teacher remains frozen throughout training.

        For vision models, we use ResNet-20 and ResNet-32 for CIFAR and ResNet-18~\citep{he2016deep} for ImageNet-1K. To assess architectural generalization, we additionally include the Vision Transformer--based DeiT-Tiny~\citep{touvron2021training}. For NLP experiments, we use BERT$_\text{BASE}$~\citep{devlin2019bert}.

        We consider bit-width configurations ranging from 1-bit to 8-bit, and quantize both weights and activations in all experiments. Vision models are trained with an EWGS-based QAT framework~\citep{lee2021network}, while Transformer-based models adopt an STE-based QAT framework~\citep{bengio2013estimating}. Despite these differences, we apply the same student-aware projection strategy across all settings. Implementation details are provided in Appendix~\ref{appendix:implement}.
        
    \subsection{Image Classification Results}
        \Cref{table:cifar10_resnet,table:cifar100_resnet,table:imagenet_resnet18} report image classification results on CIFAR-10, CIFAR-100, and ImageNet-1K. For CIFAR, we evaluate 1-, 2-, and 4-bit configurations, while for ImageNet-1K we consider 1-, 3-, 4-, and 8-bit settings. We compare SQuaT against the supervised QAT method EWGS and the label-free QAT with KD baseline SQAKD.
        
        On CIFAR-10, SQuaT consistently achieves the best performance across all bit-widths and model variants. Under 1-bit quantization with ResNet-20, SQuaT improves over SQAKD by +0.45 pp, indicating that the proposed student-aware projection remains effective even under severe quantization constraints. The improvements are particularly evident in extremely low-bit settings such as W1A1, where representational capacity is heavily restricted and mismatches between teacher and student representations become more pronounced.
        
        A similar trend is observed on CIFAR-100: under 1-bit quantization with ResNet-32, SQuaT improves over SQAKD by +1.74 pp, with gains particularly pronounced in extreme low-bit settings. Across different model sizes and quantization configurations, SQuaT consistently achieves either the best or competitive performance, demonstrating stable improvements over the baseline.
        
        On the large-scale ImageNet-1K benchmark, SQuaT again shows positive gains across all evaluated settings. Under 1-bit quantization, it improves over SQAKD by +0.40 pp. Although modest in absolute magnitude, such gains are non-trivial given the scale of the dataset and the difficulty of extreme quantization. These results indicate that the benefit of student-aware projection persists even in large-scale training settings where optimization becomes significantly more challenging.
        
        Overall, SQuaT consistently improves performance across datasets of varying scale and complexity, as well as across different bit-widths and model sizes. The performance gap widens in extreme low-bit settings, highlighting the robustness of the proposed projection strategy under stringent quantization constraints. 
        Notably, despite operating in a strictly label-free setting, SQuaT outperforms the supervised QAT baseline EWGS in several configurations. Furthermore, under 4-bit quantization on CIFAR-10 and CIFAR-100, SQuaT even surpasses the corresponding pretrained full-precision models. These results suggest that aligning teacher features to the student’s quantization lattice yields systematic advantages over prior projection-based distillation baselines and can produce edge-friendly models with strong accuracy.

        \begin{table*}[ht]
            \caption{Top-1 Test Accuracy (\%) of ResNet-20 and ResNet-32 on CIFAR-10.}
            \label{table:cifar10_resnet}
            \centering
            \begin{tabular}{c | ccc | ccc}
                \toprule
                \multirow{2}{*}{\textbf{Method}} 
                & \multicolumn{3}{c}{\textbf{ResNet-20 \scriptsize{(FP: 92.63)}}} 
                & \multicolumn{3}{c}{\textbf{ResNet-32 \scriptsize{(FP: 93.71)}}} \\
                \cmidrule(lr){2-4} \cmidrule(lr){5-7}
                & \textbf{W1A1} & \textbf{W2A2} & \textbf{W4A4} & \textbf{W1A1} & \textbf{W2A2} & \textbf{W4A4} \\
                \midrule
                
                \footnotesize{EWGS}
                & \footnotesize{86.42}{\scriptsize\textcolor{gray}{$\pm$0.01}} 
                & \footnotesize{91.41}{\scriptsize\textcolor{gray}{$\pm$0.04}} 
                & \footnotesize{92.49}{\scriptsize\textcolor{gray}{$\pm$0.09}} 
                
                & \footnotesize{86.56}{\scriptsize\textcolor{gray}{$\pm$0.03}} 
                & \footnotesize{92.89}{\scriptsize\textcolor{gray}{$\pm$0.05}} 
                & \footnotesize{93.75}{\scriptsize\textcolor{gray}{$\pm$0.03}} \\ 
                \midrule
                
                \textbf{SQAKD}\footnotesize{(EWGS)}
                & 86.52{\scriptsize\textcolor{gray}{$\pm$0.05}} 
                & 91.65{\scriptsize\textcolor{gray}{$\pm$0.13}} 
                & 92.60{\scriptsize\textcolor{gray}{$\pm$0.01}} 

                & 88.10{\scriptsize\textcolor{gray}{$\pm$0.11}} 
                & 92.90{\scriptsize\textcolor{gray}{$\pm$0.12}} 
                & 93.59{\scriptsize\textcolor{gray}{$\pm$0.02}} \\ 

                \textbf{SQuaT}\footnotesize{(EWGS)}
                & \textbf{86.97}{\scriptsize\textcolor{gray}{$\pm$0.14}} 
                & \textbf{91.90}{\scriptsize\textcolor{gray}{$\pm$0.11}} 
                & \textbf{92.67}{\scriptsize\textcolor{gray}{$\pm$0.03}} 

                & \textbf{88.17}{\scriptsize\textcolor{gray}{$\pm$0.04}} 
                & \textbf{93.16}{\scriptsize\textcolor{gray}{$\pm$0.09}} 
                & \textbf{93.72}{\scriptsize\textcolor{gray}{$\pm$0.01}} \\ 

                \textcolor{orange}{$\Delta$} \footnotesize{(SQuaT - SQAKD)}
                & \footnotesize\textcolor{orange}{+0.45} 
                & \footnotesize\textcolor{orange}{+0.35} 
                & \footnotesize\textcolor{orange}{+0.07} 

                & \footnotesize\textcolor{orange}{+0.07} 
                & \footnotesize\textcolor{orange}{+0.26} 
                & \footnotesize\textcolor{orange}{+0.13} \\ 
            
            \bottomrule
            \end{tabular}
        \end{table*}
        
        \begin{table*}[ht]
            \caption{Top-1 Test Accuracy (\%) of ResNet-20 and ResNet-32 on CIFAR-100.}
            \label{table:cifar100_resnet}
            \centering
            \begin{tabular}{c | ccc | ccc}
                \toprule
                \multirow{2}{*}{\textbf{Method}} 
                & \multicolumn{3}{c}{\textbf{ResNet-20 \scriptsize{(FP: 68.89)}}} 
                & \multicolumn{3}{c}{\textbf{ResNet-32 \scriptsize{(FP: 71.40)}}} \\
                \cmidrule(lr){2-4} \cmidrule(lr){5-7}
                & \textbf{W1A1} & \textbf{W2A2} & \textbf{W4A4} & \textbf{W1A1} & \textbf{W2A2} & \textbf{W4A4} \\
                \midrule
                
                \footnotesize{EWGS}
                & \footnotesize{56.86}{\scriptsize\textcolor{gray}{$\pm$0.24}} 
                & \footnotesize{66.68}{\scriptsize\textcolor{gray}{$\pm$0.16}} 
                & \footnotesize{68.49}{\scriptsize\textcolor{gray}{$\pm$0.03}} 
                
                & \footnotesize{59.76}{\scriptsize\textcolor{gray}{$\pm$0.44}} 
                & \footnotesize{69.12}{\scriptsize\textcolor{gray}{$\pm$0.35}} 
                & \footnotesize{70.60}{\scriptsize\textcolor{gray}{$\pm$0.14}} \\ 
                \midrule
                
                \textbf{SQAKD}\footnotesize{(EWGS)}
                & 56.38{\scriptsize\textcolor{gray}{$\pm$0.03}} 
                & 66.76{\scriptsize\textcolor{gray}{$\pm$0.29}} 
                & 69.13{\scriptsize\textcolor{gray}{$\pm$0.01}} 

                & 59.38{\scriptsize\textcolor{gray}{$\pm$0.04}} 
                & 70.00{\scriptsize\textcolor{gray}{$\pm$0.01}} 
                & 71.65{\scriptsize\textcolor{gray}{$\pm$0.01}} \\ 
                
                \textbf{SQuaT}\footnotesize{(EWGS)}
                & \textbf{56.74}{\scriptsize\textcolor{gray}{$\pm$0.12}} 
                & \textbf{67.35}{\scriptsize\textcolor{gray}{$\pm$0.01}} 
                & \textbf{69.15}{\scriptsize\textcolor{gray}{$\pm$0.02}} 

                & \textbf{61.12}{\scriptsize\textcolor{gray}{$\pm$0.19}} 
                & \textbf{70.57}{\scriptsize\textcolor{gray}{$\pm$0.04}} 
                & \textbf{71.94}{\scriptsize\textcolor{gray}{$\pm$0.05}} \\ 

                \textcolor{orange}{$\Delta$} \footnotesize{(SQuaT - SQAKD)}
                & \footnotesize\textcolor{orange}{+0.36} 
                & \footnotesize\textcolor{orange}{+0.59} 
                & \footnotesize\textcolor{orange}{+0.02} 

                & \footnotesize\textcolor{orange}{+1.74} 
                & \footnotesize\textcolor{orange}{+0.57} 
                & \footnotesize\textcolor{orange}{+0.29} \\ 
            
            \bottomrule
            \end{tabular}
        \end{table*}
        
        \begin{table}[ht]
            \caption{Top-1 Test Accuracy (\%) of ResNet-18 on ImageNet-1K.}
            \label{table:imagenet_resnet18}
            \centering
            \begin{tabular}{c | ccc}
                \toprule
                \multirow{2}{*}{\textbf{Bit-width}} 
                & \multicolumn{3}{c}{\textbf{ResNet-18 \scriptsize{(FP: 69.76)}}} \\
                \cmidrule(lr){2-4}
                & \textbf{SQAKD} & \textbf{SQuaT} & \textcolor{orange}{$\Delta$} \\
                \midrule
                
                \textbf{W1A1}
                & 49.46{\scriptsize\textcolor{gray}{$\pm$0.01}} 
                & \textbf{49.86}{\scriptsize\textcolor{gray}{$\pm$0.25}} 
                & \footnotesize\textcolor{orange}{+0.40} \\ 
                
                \textbf{W3A3}
                & 67.62{\scriptsize\textcolor{gray}{$\pm$0.11}} 
                & \textbf{67.81}{\scriptsize\textcolor{gray}{$\pm$0.06}} 
                & \footnotesize\textcolor{orange}{+0.19} \\ 

                \textbf{W4A4}
                & 68.56{\scriptsize\textcolor{gray}{$\pm$0.13}}  
                & \textbf{68.77}{\scriptsize\textcolor{gray}{$\pm$0.07}} 
                & \footnotesize\textcolor{orange}{+0.21} \\ 
                
                \textbf{W8A8}
                & 69.25{\scriptsize\textcolor{gray}{$\pm$0.01}}  
                & \textbf{69.27}{\scriptsize\textcolor{gray}{$\pm$0.01}} 
                & \footnotesize\textcolor{orange}{+0.02} \\ 
            
            \bottomrule
            \end{tabular}
        \end{table}        
        
    \subsection{Generality Across Model Designs and Training Settings}
        We evaluate whether SQuaT consistently improves performance across different model architectures, domains, loss functions, and quantization methods. Specifically, we examine whether the gains of SQuaT persist when varying (i) model architecture (CNN, Vision Transformer, and Transformer), (ii) task domain (vision and NLP), (iii) the feature distillation loss, and (iv) the quantization method and training framework.
        
        Table~\ref{table:cifar_deit} presents results on the Vision Transformer--based DeiT-Tiny. Unlike the CNN-based ResNet family, DeiT-Tiny adopts a self-attention architecture. Nevertheless, SQuaT consistently outperforms SQAKD across all bit-width settings, indicating that the proposed student-aware projection strategy is not tied to convolutional inductive biases and generalizes to structurally distinct architectures.
        
        Table~\ref{table:glue_bert} reports results on BERT$_\text{BASE}$ evaluated on the GLUE benchmark (RTE, SST-2, QNLI). Across all tasks and bit-width configurations, SQuaT achieves higher performance than the baseline. These results demonstrate that the proposed approach extends beyond vision classification and remains effective for Transformer-based NLP models, suggesting that the benefits of student-aware projection are not limited to convolutional networks or vision tasks.
        
        In Table~\ref{table:fd_loss}, we investigate whether SQuaT depends on a specific feature distillation loss by comparing multiple loss functions, including $L_1$, $L_2$, KL divergence, and cosine similarity. Across all loss configurations, SQuaT consistently yields improvements over the baseline, indicating that the method is not specialized to a particular loss formulation.
        
        These experiments are conducted under diverse quantization frameworks and quantizer designs. The CNN-based CIFAR/ImageNet results in \Cref{table:cifar10_resnet,table:cifar100_resnet,table:imagenet_resnet18} use uniform quantization with EWGS-based QAT. In contrast, the DeiT-Tiny experiments adopt STATSQ~\citep{liu2023oscillation} and LSQ~\citep{esser2019lsq}, while the NLP experiments employ TwnQuantizer~\citep{li2016ternary} and SymQuantizer~\citep{shen2020q}. All Transformer-based models are trained using STE-based QAT.
        
        Despite these heterogeneous quantization settings, SQuaT consistently improves performance. These observations suggest that the effectiveness of SQuaT does not rely on a particular architectural assumption or quantization implementation, but rather stems from the general principle of aligning teacher features to the student’s quantization lattice.
        
        \begin{table*}[ht]
            \caption{Top-1 Test Accuracy (\%) of DeiT-Tiny on CIFAR-10 and CIFAR-100.}
            \label{table:cifar_deit}
            \centering
            \begin{tabular}{c | ccc | ccc}
                \toprule
                \multirow{2}{*}{\textbf{Method}} 
                & \multicolumn{3}{c}{\textbf{CIFAR-10 \scriptsize{(FP: 97.26)}}} 
                & \multicolumn{3}{c}{\textbf{CIFAR-100 \scriptsize{(FP: 97.33)}}} \\
                \cmidrule(lr){2-4} \cmidrule(lr){5-7}
                & \textbf{W2A2} & \textbf{W3A3} & \textbf{W4A4} & \textbf{W2A2} & \textbf{W3A3} & \textbf{W4A4} \\
                \midrule                
                
                \textbf{SQAKD}
                & 85.78{\scriptsize\textcolor{gray}{$\pm$0.11}} 
                & 85.81{\scriptsize\textcolor{gray}{$\pm$0.09}} 
                & 85.88{\scriptsize\textcolor{gray}{$\pm$0.01}} 

                & 86.39{\scriptsize\textcolor{gray}{$\pm$0.01}} 
                & 86.21{\scriptsize\textcolor{gray}{$\pm$0.16}} 
                & 86.78{\scriptsize\textcolor{gray}{$\pm$0.15}} \\ 
                
                \textbf{SQuaT}
                & \textbf{86.67}{\scriptsize\textcolor{gray}{$\pm$0.23}} 
                & \textbf{87.50}{\scriptsize\textcolor{gray}{$\pm$0.02}} 
                & \textbf{87.98}{\scriptsize\textcolor{gray}{$\pm$0.01}} 

                & \textbf{87.66}{\scriptsize\textcolor{gray}{$\pm$0.23}} 
                & \textbf{89.48}{\scriptsize\textcolor{gray}{$\pm$0.51}} 
                & \textbf{90.20}{\scriptsize\textcolor{gray}{$\pm$0.23}} \\ 

                \textcolor{orange}{$\Delta$}
                & \footnotesize\textcolor{orange}{+0.89} 
                & \footnotesize\textcolor{orange}{+1.69} 
                & \footnotesize\textcolor{orange}{+2.10} 

                & \footnotesize\textcolor{orange}{+1.27} 
                & \footnotesize\textcolor{orange}{+3.27} 
                & \footnotesize\textcolor{orange}{+3.42} \\ 
            
            \bottomrule
            \end{tabular}
        \end{table*}
        
        \begin{table*}[ht]
            \caption{Performance (\%) of BERT$_\text{BASE}$ on the GLUE Benchmark.}
            \label{table:glue_bert}
            \centering
            \begin{tabular}{c|c|ccc}
                \toprule
                \multirow{2}{*}{\textbf{Bit-width}}
                & \multirow{2}{*}{\textbf{Method}}
                & {\textbf{RTE}}
                & {\textbf{SST-2}}
                & {\textbf{QNLI}} \\
                &
                & \footnotesize 2.5k \scriptsize{(FP: 68.95)} & \footnotesize 67k \scriptsize{(FP: 93.23)} & \footnotesize 108k \scriptsize{(FP: 91.25)} \\
                \midrule

                \multirow{3}{*}{\textbf{W3A3}}
                & SQAKD
                & 54.15{\scriptsize\textcolor{gray}{$\pm$0.51}} 
                & 86.24{\scriptsize\textcolor{gray}{$\pm$0.16}} 
                & 81.88{\scriptsize\textcolor{gray}{$\pm$0.13}} \\ 
                
                & SQuaT
                & \textbf{57.22}{\scriptsize\textcolor{gray}{$\pm$0.25}} 
                & \textbf{86.93}{\scriptsize\textcolor{gray}{$\pm$0.16}} 
                & \textbf{83.14}{\scriptsize\textcolor{gray}{$\pm$0.06}} \\ 
                
                & \textcolor{orange}{$\Delta$}
                & \footnotesize\textcolor{orange}{+3.07} 
                & \footnotesize\textcolor{orange}{+0.69} 
                & \footnotesize\textcolor{orange}{+1.26} \\ 

                \midrule
                
                \multirow{3}{*}{\textbf{W4A4}}
                & SQAKD
                & 57.40{\scriptsize\textcolor{gray}{$\pm$1.02}} 
                & 92.44{\scriptsize\textcolor{gray}{$\pm$0.16}} 
                & 90.01{\scriptsize\textcolor{gray}{$\pm$0.16}} \\ 
                
                & SQuaT
                & \textbf{65.88}{\scriptsize\textcolor{gray}{$\pm$0.25}} 
                & \textbf{92.72}{\scriptsize\textcolor{gray}{$\pm$0.08}} 
                & \textbf{90.21}{\scriptsize\textcolor{gray}{$\pm$0.03}} \\ 
                
                & \textcolor{orange}{$\Delta$}
                & \footnotesize\textcolor{orange}{+8.48} 
                & \footnotesize\textcolor{orange}{+0.28} 
                & \footnotesize\textcolor{orange}{+0.20} \\ 

                \midrule
                
                \multirow{3}{*}{\textbf{W8A8}}
                & SQAKD
                & 65.70{\scriptsize\textcolor{gray}{$\pm$0.51}} 
                & 93.41{\scriptsize\textcolor{gray}{$\pm$0.08}} 
                & 91.53{\scriptsize\textcolor{gray}{$\pm$0.02}} \\ 
                
                & SQuaT
                & \textbf{69.13}{\scriptsize\textcolor{gray}{$\pm$0.25}} 
                & \textbf{93.70}{\scriptsize\textcolor{gray}{$\pm$0.16}} 
                & \textbf{91.86}{\scriptsize\textcolor{gray}{$\pm$0.01}} \\ 
                
                & \textcolor{orange}{$\Delta$}
                & \footnotesize\textcolor{orange}{+3.43} 
                & \footnotesize\textcolor{orange}{+0.29} 
                & \footnotesize\textcolor{orange}{+0.33} \\ 
            
            \bottomrule
            \end{tabular}
        \end{table*}

        \begin{table}[ht]
            \caption{Top-1 Test Accuracy (\%) of ResNet-20 on CIFAR-10 with various feature loss types.}
            \label{table:fd_loss}
            \centering
            \begin{tabular}{c | ccc}
                \toprule
                \textbf{Loss Func.}
                & \textbf{W1A1} & \textbf{W2A2} & \textbf{W4A4} \\
                \midrule
                
                \textbf{SQAKD}
                & 86.47{\scriptsize\textcolor{gray}{$\pm$0.28}} 
                & 91.62{\scriptsize\textcolor{gray}{$\pm$0.03}} 
                & 92.59{\scriptsize\textcolor{gray}{$\pm$0.02}} \\ 
                \midrule
                
                \textbf{SQuaT$_{L_1}$}
                & 86.78{\scriptsize\textcolor{gray}{$\pm$0.14}} 
                & 91.77{\scriptsize\textcolor{gray}{$\pm$0.09}} 
                & 92.64{\scriptsize\textcolor{gray}{$\pm$0.03}} \\ 
                
                \textbf{SQuaT$_{L_2}$}
                & \textcolor{orange}{\textbf{86.97}}{\scriptsize\textcolor{gray}{$\pm$0.14}} 
                & \textcolor{orange}{\textbf{91.90}}{\scriptsize\textcolor{gray}{$\pm$0.11}} 
                & 92.67{\scriptsize\textcolor{gray}{$\pm$0.03}} \\ 

                \textbf{SQuaT$_\text{KL}$}
                & 86.74{\scriptsize\textcolor{gray}{$\pm$0.15}} 
                & 91.74{\scriptsize\textcolor{gray}{$\pm$0.03}} 
                & \textcolor{orange}{\textbf{92.70}}{\scriptsize\textcolor{gray}{$\pm$0.05}} \\ 
                
                \textbf{SQuaT$_\text{Cos}$}
                & 86.74{\scriptsize\textcolor{gray}{$\pm$0.26}} 
                & 91.66{\scriptsize\textcolor{gray}{$\pm$0.05}} 
                & 92.66{\scriptsize\textcolor{gray}{$\pm$0.03}} \\ 
            
            \bottomrule
            \end{tabular}
        \end{table}

    \subsection{Empirical Analysis of Student-Aware Projection}
        In this section, we conduct both quantitative and qualitative analyses of the proposed student-aware projection strategy. Table~\ref{table:projection_strategy} compares different projection schemes used in feature distillation. Non-aware directly distills full-precision teacher features without any alignment. Teacher-aware, following QFD~\citep{zhu2023quantized}, attaches a quantizer to the teacher, calibrates it over several epochs, and then projects the teacher features using the calibrated quantizer. In contrast, Student-aware projects teacher features onto the student’s quantization lattice before performing distillation, as in SQuaT.
        
        As shown in Table~\ref{table:projection_strategy}, the Student-aware scheme achieves higher accuracy than both Non-aware and Teacher-aware approaches across all bit-widths and model configurations. Notably, Teacher-aware projection often performs worse than the Non-aware variant. This suggests that a quantizer calibrated independently on the teacher side, without accounting for the student’s representational constraints, may still leave a mismatch between teacher and student representations. In contrast, explicitly aligning teacher features to the student’s representational space proves more effective than either directly using full-precision features or relying on teacher-side calibration.

        Figure~\ref{fig:cifar_loss_acc} shows the learning curves of feature distillation loss and Top-1 test accuracy for the three projection schemes. The Student-aware approach converges to the lowest feature distillation loss throughout training, indicating stable optimization. In contrast, the Teacher-aware scheme converges to a higher loss than even the Non-aware scheme, revealing a clear limitation. The Non-aware scheme initially decreases faster but later exhibits an increase in feature distillation loss as training progresses, suggesting a mismatch between full-precision teacher features and quantized student representations. As a result, the Student-aware approach achieves the highest final accuracy, whereas the others oscillate at higher loss levels or converge more slowly, ultimately leading to inferior performance across training and evaluation.

        These empirical observations are consistent with the theoretical analysis in Section~\ref{section:3.3}. Theorem~\ref{theorem:1} shows that projecting teacher features onto the student’s quantization lattice removes the irreducible mismatch. Accordingly, the superior convergence of the Student-aware scheme indicates effective removal of this lower bound, leading to more stable optimization and higher final performance.
        
        \begin{table*}[ht]
            \caption{Analysis of Feature Distillation Method.}
            \label{table:projection_strategy}
            \centering
            \begin{tabular}{c | ccc | ccc}
                \toprule
                \multirow{2}{*}{\textbf{Method}} 
                & \multicolumn{3}{c |}{\textbf{CIFAR-10 / ResNet-20 \scriptsize{(FP: 92.63)}}} 
                & \multicolumn{3}{c}{\textbf{CIFAR-100 / ResNet-32 \scriptsize{(FP: 71.40)}}} \\
                \cmidrule(lr){2-4} \cmidrule(lr){5-7}
                & \textbf{W1A1} & \textbf{W2A2} & \textbf{W4A4} & \textbf{W1A1} & \textbf{W2A2} & \textbf{W4A4} \\
                \midrule
                
                \textbf{Non-aware}
                & \footnotesize{86.30}{\scriptsize\textcolor{gray}{$\pm$0.12}} 
                & \footnotesize{91.73}{\scriptsize\textcolor{gray}{$\pm$0.13}} 
                & \footnotesize{92.64}{\scriptsize\textcolor{gray}{$\pm$0.02}} 
                
                & 60.77{\scriptsize\textcolor{gray}{$\pm$0.16}} 
                & 70.49{\scriptsize\textcolor{gray}{$\pm$0.04}} 
                & 71.69{\scriptsize\textcolor{gray}{$\pm$0.07}} \\ 
                
                \textbf{Teacher-aware}
                & 86.44{\scriptsize\textcolor{gray}{$\pm$0.18}} 
                & 91.59{\scriptsize\textcolor{gray}{$\pm$0.09}} 
                & 92.44{\scriptsize\textcolor{gray}{$\pm$0.05}} 

                & 59.59{\scriptsize\textcolor{gray}{$\pm$0.24}} 
                & 69.50{\scriptsize\textcolor{gray}{$\pm$0.15}} 
                & 71.14{\scriptsize\textcolor{gray}{$\pm$0.13}} \\ 
                
                \textbf{Student-aware}
                & \textcolor{orange}{\textbf{86.97}}{\scriptsize\textcolor{gray}{$\pm$0.14}} 
                & \textcolor{orange}{\textbf{91.90}}{\scriptsize\textcolor{gray}{$\pm$0.11}} 
                & \textcolor{orange}{\textbf{92.67}}{\scriptsize\textcolor{gray}{$\pm$0.03}} 

                & \textcolor{orange}{\textbf{61.02}}{\scriptsize\textcolor{gray}{$\pm$0.19}} 
                & \textcolor{orange}{\textbf{70.57}}{\scriptsize\textcolor{gray}{$\pm$0.04}} 
                & \textcolor{orange}{\textbf{71.94}}{\scriptsize\textcolor{gray}{$\pm$0.05}} \\ 
            
            \bottomrule
            \end{tabular}
        \end{table*}

        \begin{figure*}[ht]
            \centering
            \includegraphics[width=\textwidth]{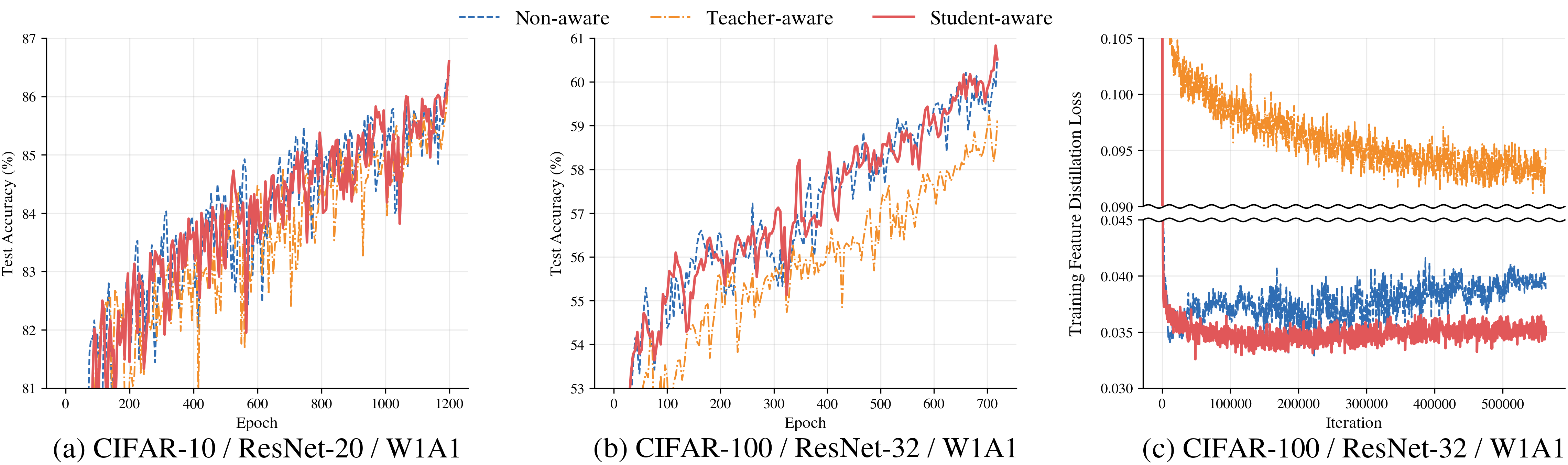}
            \caption{Training curves of feature distillation loss and Top-1 test accuracy under different bit-width settings.}
            \label{fig:cifar_loss_acc}
        \end{figure*}

    \subsection{Sensitivity to Loss Weight}
        We analyze the sensitivity of SQuaT to the feature distillation loss weight under 4-bit quantization on CIFAR-100 with ResNet-32 (see Table~\ref{table:cifar100_loss_weight}). Following the weighting scheme of SQAKD, we fix the logit loss weight to $\lambda_{\text{logit}}=100$ and vary the feature loss weight $\lambda_{\text{feat}}$. Among the evaluated configurations, $\lambda_{\text{feat}}=10$ achieves the best accuracy of 71.94\%. Accordingly, we use $\lambda_{\text{logit}}=100$ and $\lambda_{\text{feat}}=10$ for the CIFAR-100 ResNet-32 4-bit setting, while all other experiments set $\lambda_{\text{logit}}=\lambda_{\text{feat}}=1$.
            
        \begin{table}[ht]
            \centering
            \caption{Sensitivity to the feature loss weight $\lambda_{\text{feat}}$ on CIFAR-100 with ResNet-32 under the 4-bit setting.}
            \label{table:cifar100_loss_weight}
            \begin{tabular}{c|c|c|c}
                \toprule
                $\lambda_\mathrm{feat}$ & Top-1 Acc. & $\lambda_\mathrm{feat}$ & Top-1 Acc. \\ 
                \midrule
                \textbf{10}  & \textcolor{orange}{\textbf{71.94}} $\pm$ 0.05 & \textbf{60}  & 71.77 $\pm$ 0.14 \\
                \textbf{20}  & 71.86 $\pm$ 0.06 & \textbf{70}  & 71.76 $\pm$ 0.06 \\
                \textbf{30}  & 71.71 $\pm$ 0.03 & \textbf{80}  & 71.78 $\pm$ 0.06 \\
                \textbf{40}  & 71.75 $\pm$ 0.05 & \textbf{90}  & 71.79 $\pm$ 0.04 \\
                \textbf{50}  & 71.73 $\pm$ 0.03 & \textbf{100} & 71.78 $\pm$ 0.01 \\
                \bottomrule
            \end{tabular}
        \end{table}

    \subsection{On-Device Evaluation}
        To evaluate practical efficiency, we measure inference performance on an NVIDIA Jetson Nano using ResNet-18 trained on ImageNet (see Table~\ref{table:imagenet_jetson}). Compared with FP32 inference, INT8 inference increases throughput from 294.32 FPS to 2046.03 FPS and reduces latency from 3.44 ms to 0.53 ms, corresponding to a 6.49$\times$ speedup. These results demonstrate that SQuaT-trained quantized models provide substantial efficiency gains on edge devices while maintaining competitive accuracy.

        \begin{table}[ht]
            \caption{Inference Performance on Jetson Nano.} 
            \label{table:imagenet_jetson}
            \centering
            \resizebox{\linewidth}{!}{ 
            \begin{tabular}{ccccc}
                \toprule
                \multicolumn{1}{c}{\textbf{Model}}
                & {\makecell{\textbf{Bit}\\\textbf{width}}} 
                & {\makecell{\textbf{Throughput}\\\textbf{(FPS)}}}
                & {\makecell{\textbf{Latency}\\\textbf{(ms)}}}
                & \textbf{Speedup} \\
                \midrule
                \multirow{2}{*}{ResNet-18}
                    & FP32 & 294.32 & 3.442 & - \\
                    & INT8 & 2046.03 & 0.53 & 6.49$\times$ \\
                \bottomrule
            \end{tabular}}
        \end{table}

\section{CONCLUSION}
    We identify a fundamental limitation in label-free QAT with KD: prior teacher-aware projection suffers from a distributional mismatch between full-precision teacher features and quantized student representations, inducing an unattainable residual and an irreducible lower bound on the distillation loss. To address this issue, we introduce SQuaT, which applies a student-aware projection to align teacher features with the student’s quantization lattice, thereby eliminating this bound by restricting targets to the student’s attainable set. Experiments demonstrate consistent improvements over prior label-free QAT with KD methods, particularly in extreme low-bit regimes, and the approach is agnostic to model components, making it broadly applicable.
    
\clearpage

\subsubsection*{Acknowledgements}
This work was supported by Hyundai Motor Company and Kia, and partially supported by the Institute of Information \& Communications Technology Planning \& Evaluation(IITP) grant funded by the Korea government(MSIT) (o.RS-2025-02219317, AI Star Fellowship(Kookmin University)), and the National Research Foundation of Korea(NRF) grant funded by the Korea government(MSIT) (RS-2025-23524783)

\bibliography{references}

\clearpage
\appendix
\thispagestyle{empty}

\onecolumn
\aistatstitle{Supplementary Materials: Self-Supervised Knowledge Distillation via Student-Aware Quantized Teacher Features}

\section{IMPLEMENTATION DETAILS}
\label{appendix:implement}
On CIFAR-10~\citep{krizhevsky2009learning}, ResNet-20 and ResNet-32~\citep{he2016deep} were trained for 1200 epochs with a batch size of 256. The learning rates were set to $1 \times 10^{-3}$ for model parameters and $1 \times 10^{-5}$ for quantization parameters, and a weight decay of $1 \times 10^{-4}$ was applied. On CIFAR-100~\citep{krizhevsky2009learning}, the same architectures were trained for 720 epochs with a batch size of 64. The learning rates were set to $5 \times 10^{-4}$ for model parameters and $5 \times 10^{-6}$ for quantization parameters, and a weight decay of $5 \times 10^{-4}$ was used.

DeiT-Tiny~\citep{touvron2021training} was also evaluated on the same datasets. The model was trained for 300 epochs with a batch size of 128. The learning rates for both the model parameters and quantization parameters were set to $1 \times 10^{-3}$, and a weight decay of $5 \times 10^{-2}$ was applied.

On ImageNet-1K~\citep{deng2009imagenet}, ResNet-18~\citep{he2016deep} was trained for 100 epochs with a batch size of 128. The model was initialized with pretrained weights provided by PyTorch~\citep{paszke2019pytorch} . The learning rate was set to $1 \times 10^{-2}$ for model parameters and $1 \times 10^{-5}$ for quantization parameters, with a weight decay of $1 \times 10^{-4}$.

For BERT$_\text{BASE}$~\citep{devlin2019bert}, experiments were conducted on the GLUE benchmark~\citep{wang2018glue} for 3 epochs with a batch size of 32. The learning rate was set to $2 \times 10^{-5}$ for both the model parameters and quantization parameters, and a weight decay of $1 \times 10^{-2}$ was applied. The model was initialized with pretrained weights provided by HuggingFace.

ResNet-based models were optimized using the Adam optimizer~\citep{kingma2014adam}. However, when training ResNet-18 on ImageNet-1K, the model parameters were optimized using Stochastic Gradient Descent (SGD), while the quantization parameters were optimized using Adam. For DeiT-Tiny and BERT$_\text{BASE}$, AdamW~\citep{loshchilov2017decoupled} and BertAdam optimizers were used, respectively. 

The feature locations where SQuaT is applied depend on the model architecture. For CNN-based models (ResNet), SQuaT is applied to the input feature of the last convolutional layer in the final residual block to perform distillation in the quantized space. For Transformer-based models (DeiT, BERT), SQuaT is applied to the input feature of the first linear layer in the FFN of the last Transformer block (i.e., the attention output), with distillation performed over all tokens.

The implementation was based on Python 3.11.11 and PyTorch 2.8.0, and all experiments were conducted on an NVIDIA RTX A6000 GPU. 

\begin{table}[ht]
    \centering
    \caption{Implementation Details}
    \label{tab:training_config}
    \footnotesize
    \setlength{\tabcolsep}{4pt}
    \begin{tabular}{llcccccc}
        \toprule
        \textbf{Dataset} & \textbf{Model} & \textbf{Epoch} & \textbf{Batch} & \textbf{Optimizer} & \textbf{LR$_m$} & \textbf{LR$_q$}  & \textbf{Weight Decay} \\
        \midrule
        
        \multirow{3}{*}{CIFAR-10}
        & ResNet-20  & 1200 & 256 & Adam & $1\times10^{-3}$ & $1\times10^{-5}$ & $1\times10^{-4}$ \\
        & ResNet-32  & 1200 & 256 & Adam & $1\times10^{-3}$ & $1\times10^{-5}$ & $1\times10^{-4}$ \\
        & DeiT-Tiny  & 300  & 128 & AdamW & $1\times10^{-3}$ & $1\times10^{-3}$ & $5\times10^{-2}$  \\
        \midrule
        
        \multirow{3}{*}{CIFAR-100}
        & ResNet-20  & 720  & 64  & Adam & $5\times10^{-4}$ & $5\times10^{-6}$ & $5\times10^{-4}$ \\
        & ResNet-32  & 720  & 64  & Adam & $5\times10^{-4}$ & $5\times10^{-6}$ & $5\times10^{-4}$ \\
        & DeiT-Tiny  & 300  & 128 & AdamW & $1\times10^{-3}$ & $1\times10^{-3}$ & $5\times10^{-2}$ \\
        \midrule
        
        ImageNet-1K
        & ResNet-18  & 100  & 128 & Adam, SGD & $1\times10^{-2}$ & $1\times10^{-5}$ & $1\times10^{-4}$ \\
        \midrule
        
        GLUE
        & BERT$_{\text{BASE}}$ & 3 & 32 & BertAdam & $2\times10^{-5}$ & $2\times10^{-5}$ & $1\times10^{-2}$ \\
    \bottomrule
    \end{tabular}
\end{table}

\section{ADDITIONAL EXPERIMENTS}
Table~\ref{table:tinyimagenet_resnet18} reports the image classification results on Tiny-ImageNet~\citep{le2015tiny}. In this experiment, we evaluate 1-, 3-, and 4-bit quantization settings, and the models are trained using an EWGS-based QAT framework~\citep{lee2021network}. Tiny-ImageNet contains 200 classes with large intra-class variation, making it more complex than CIFAR-100. Despite this increased complexity, SQuaT consistently outperforms SQAKD~\citep{zhao2024self} across most bit-width settings.

In particular, under 1-bit quantization, SQuaT achieves a +0.34 pp improvement over SQAKD, which represents the largest performance gain among all bit-width configurations. This result demonstrates that the proposed method remains effective even under extreme low-bit quantization constraints.

These results indicate that SQuaT maintains stable performance improvements on larger-scale datasets, and a similar trend is consistently observed in large-scale ImageNet-1K experiments, as shown in Table~\ref{table:imagenet_resnet18}.

\begin{table}[ht]
    \caption{Top-1 Test Accuarcy (\%) of ResNet-18 on Tiny-ImageNet.}
    \label{table:tinyimagenet_resnet18}
    \centering
    \begin{tabular}{c | ccc}
        \toprule
        \multirow{2}{*}{\textbf{Bit-width}} 
        & \multicolumn{3}{c}{\textbf{ResNet-18 \scriptsize{(FP: 65.55)}}} \\
        \cmidrule(lr){2-4}
        & \textbf{SQAKD} & \textbf{SQuaT} & \textcolor{orange}{$\Delta$} \\
        \midrule
        
        \textbf{W1A1}
        & 59.18{\scriptsize\textcolor{gray}{$\pm$0.04}} 
        & \textbf{59.52}{\scriptsize\textcolor{gray}{$\pm$0.03}} 
        & \footnotesize\textcolor{orange}{+0.34} \\ 
        
        \textbf{W3A3}
        & 65.79{\scriptsize\textcolor{gray}{$\pm$0.15}} 
        & \textbf{65.59}{\scriptsize\textcolor{gray}{$\pm$0.02}} 
        & \footnotesize\textcolor{orange}{-0.20} \\ 

        \textbf{W4A4}
        & 65.73{\scriptsize\textcolor{gray}{$\pm$0.11}}  
        & \textbf{65.88}{\scriptsize\textcolor{gray}{$\pm$0.03}} 
        & \footnotesize\textcolor{orange}{+0.15} \\ 
        
        \textbf{W8A8}
        & 65.88{\scriptsize\textcolor{gray}{$\pm$0.01}}  
        & \textbf{66.05}{\scriptsize\textcolor{gray}{$\pm$0.05}} 
        & \footnotesize\textcolor{orange}{+0.17} \\ 
    
    \bottomrule
    \end{tabular}
\end{table}  

\end{document}